\documentclass[11pt,a4paper]{article}

\usepackage[latin1]{inputenc}
\usepackage[american]{babel}      
\usepackage[T1]{fontenc}
\usepackage{amsmath,amssymb,amsthm}
\usepackage{graphicx}
\usepackage{subfigure}
\usepackage{multirow}

\usepackage{color}
\DeclareGraphicsExtensions{.eps,.pdf,.png,.jpg}

\title{{\bf Applications of sampling Kantorovich operators to thermographic images\\ for seismic engineering}}
\author{{\bf Federico Cluni}, \hskip0.5cm {\bf Danilo Costarelli},\\
 {\bf Anna Maria Minotti} \hskip0.3cm and \hskip0.3cm {\bf Gianluca Vinti}\thanks{corresponding author, email: {\tt gianluca.vinti@unipg.it}}} 

\date{}

\newcommand{\miu}{\leq}

\newcommand{\ep}{\varepsilon}
\newcommand{\N}{\mathbb{N}}
\newcommand{\R}{\mathbb{R}}
\newcommand{\Z}{\mathbb{Z}}
\newcommand{\disp}{\displaystyle}

\newcommand{\kk}{\underline{k}}
\newcommand{\tk}{t_{\underline{k}}}
\newcommand{\uu}{\underline{u}}
\newcommand{\xx}{\underline{x}}
\newcommand{\be}{\begin{equation}}
\newcommand{\ee}{\end{equation}}
\newtheorem{definition}{Definition}[section]
\newtheorem{remark}[definition]{Remark}
\newtheorem{theorem}[definition]{Theorem}
\newtheorem{lemma}[definition]{Lemma}

\newtheorem{corollary}[definition]{Corollary}

\begin{document}
\maketitle

\begin{abstract}
In this paper, we present some applications of the multivariate sampling Kantorovich operators $S_w$ to seismic engineering. The mathematical theory of these operators, both in the space of continuous functions and in Orlicz spaces, show how it is possible to approximate/reconstruct multivariate signals, such as images. In particular, to obtain applications for thermographic images a mathematical algorithm is developed using MATLAB and matrix calculus. The setting of Orlicz spaces is important since allow us to reconstruct not necessarily continuous signals by means of $S_w$. The reconstruction of thermographic images of buildings by our sampling Kantorovich algorithm allow us to obtain models for the simulation of the behavior of structures under seismic action. We analyze a real world case study in term of structural analysis and we compare the behavior of the building under seismic action using various models.
\end{abstract}

\section{Introduction}

The sampling Kantorovich operators have been introduced to approximate and reconstruct not necessarily continuous signals. In \cite{BABUSTVI}, the authors introduced this operators starting from the well-known generalized sampling operators (see e.g. \cite{BUFIST,BUST2,BAVI2,VI1,BABUSTVI2}) and replacing, in their definition, the sample values $f(k/w)$ with $w \int_{k/w}^{(k+1)/w}f(u)\,du$. Clearly, this is the most natural mathematical modification to obtain operators which can be well-defined also for general measurable, locally integrable functions, not necessarily continuous. Moreover, this situation very often occur in Signal Processing, when one cannot match exactly the sample at the point $k/w$: this represents the so-called "time-jitter'' error. The theory of sampling Kantorovich operators allow us to reduces the time-jitter error, calculating the information in a neighborhood of $k/w$ rather that exactly in the node $k/w$. These operators, as the generalized sampling operators, represent an approximate version of classical sampling series, based on the Whittaker-Kotelnikov-Shannon sampling theorem (see e.g. \cite{ANVI}).

   Subsequently, the sampling Kantorovich operators have been studied in various settings. In \cite{COVI,COVI2} the multivariate version of these operators were introduced. Results concerning the order of approximation are shown in \cite{COVI3}. Extensions to more general contexts are presented in \cite{VIZA1,VIZA2,VIZA3,BAMA2}.

   The multivariate sampling Kantorovich operators considered in this paper are of the form:
$$
(S_w f)(\underline{x})\ :=\ \sum_{\underline{k} \in \Z^n} \chi(w\underline{x}-t_{\underline{k}})\left[\frac{w^n}{A_{\underline{k}}} \int_{R_{\underline{k}}^w}f(\underline{u})\ d\underline{u}\right], \hskip0.7cm (\underline{x} \in \R^n), \hskip0.8cm \mbox{(I)}
$$
where $f: \R^n \to \R$ is a locally integrable function such that the above series is convergent for every $\underline{x} \in \R^n$. The symbol $t_{\underline{k}}=\left(t_{k_1},...,t_{k_n}\right)$ denotes vectors where each $(t_{k_i})_{k_i \in \Z}$, $i=1,...,n$ is a certain strictly increasing sequence of real numbers with $\Delta_{k_i}=t_{k_{i+1}}-t_{k_i}>0$.
Note that, the sequences $(t_{k_i})_{k_i \in \Z}$ are not necessary equally spaced. We denote by $R_{\underline{k}}^w$ the sets:
$$
R_{\underline{k}}^w\ :=\ \left[\frac{t_{k_1}}{w},\frac{t_{k_1+1}}{w}\right]\times\left[\frac{t_{k_2}}{w},\frac{t_{k_2+1}}{w}\right]\times...\times\left[\frac{t_{k_n}}{w},\frac{t_{k_n+1}}{w}\right], \hskip2.2cm \mbox{(II)}
$$
$w>0$ and $A_{\underline{k}} = \Delta_{k_1} \cdot \Delta_{k_2} \cdot...\cdot \Delta_{k_n}$, $\underline{k} \in \Z^n$. Moreover, the function $\chi:\R^n \to \R$ is a kernel satisfying suitable assumptions. 

   For the operators in (I) we recall some convergence results. We have that the family $(S_wf)_{w>0}$ converges pointwise to $f$, when $f$ is continuous and bounded and $(S_w f)_{w>0}$ converges uniformly to $f$, when $f$ is uniformly continuous and bounded. Moreover, to cover the case of not necessarily continuous signal, we study our operators in the general setting of Orlicz spaces $L^{\varphi}(\R^n)$. For functions $f$ belonging to $L^{\varphi}(\R^n)$ and generated by the convex $\varphi$-function $\varphi$, the family of sampling Kantorovich operators is "modularly'' convergent to $f$, being the latter the natural concept of convergence in this setting.  

   The latter result, allow us to apply the theory of the sampling Kantorovich operators to approximate and reconstruct images. In fact, static gray scale images are characterized by jumps of gray levels mainly concentrated in their contours or edges and this can be translated, from a mathematical point of view, by discontinuities (see e.g. \cite{GOWO}).

  Here, we introduce and analyze in detail some practical applications of the sampling Kantorovich algorithm to thermographic images, very useful for the analysis of buildings in seismic engineering. The thermography is a remote sensing technique, performed by the image acquisition in the infrared. Thermographic images are widely used to make non-invasive investigations of structures, to analyze the story of the building wall, to make diagnosis and monitoring buildings, and to make structural measurements. A further important use, is the application of the texture algorithm for the separation between the bricks and the mortar in masonries images. Through this procedure the civil engineers becomes able to determine the mechanical parameters of the structure under investigation. Unfortunately, the direct application of the texture algorithm to the thermographic images, can produce errors, as an incorrect separation between the bricks and the mortar. 

  Then, we use the sampling Kantorovich operators to process the thermographic images before to apply the texture algorithm. In this way, the result produced by the texture becomes more refined and therefore we can apply structural analysis after the calculation of the various parameters involved. In order to show the feasibility of our applications, we present in detail a real-world case-study.


\section{Preliminaries}

In this section we recall some preliminaries, notations and definitions.

  We denote by $C(\R^n)$ (resp. $C^0(\R^n)$) the space of all uniformly continuous and bounded (resp. continuous and bounded) functions $f:\R^n\to \R$ endowed with the usual sup-norm $\|f\|_{\infty} := \sup_{\uu \in \R^n}\left|f(\uu)\right|$, $\uu = (u_1,\, ...,\, u_n)$, and by $C_c(\R^n) \subset C(\R^n)$ the subspace of the elements having compact support. Moreover, $M\left(\R^n\right)$ will denote the linear space of all (Lebesgue) measurable real functions defined on $\R^n$.

   We now recall some basic fact concerning Orlicz spaces, see e.g. \cite{MUORL,MU1,RAO1,BAVI_1,BAMUVI}.  

\noindent   The function $\varphi: \R^+_0 \to \R^+_0$ is said to be a $\varphi$-function if it satisfies the following assumptions: $(i)$ $\varphi \left(0\right)=0$, and $\varphi \left(u\right)>0$ for every $u>0$; $(ii)$ $\varphi$ is continuous and non decreasing on $\R^+_0$; $(iii)$
$\lim_{u\to \infty}\varphi(u)\ =\ + \infty$.

  The functional $I^{\varphi} : M(\R^n)\to [0,+\infty]$ (see e.g. \cite{MU1,BAMUVI}) defined by
$$
I^{\varphi} \left[f\right] := \int_{\R^n} \varphi(\left| f(\underline{x}) \right|)\ d\underline{x},\ \hskip0.5cm \left(f \in M(\R^n)\right),
$$
is a modular in $M(\R^n)$ . The Orlicz space generated by $\varphi$ is given by
$$
L^{\varphi}(\R^n) := \left\{f \in M\left(\R^n\right):\ I^{\varphi} [\lambda f]<+\infty,\ \mbox{for\ some}\ \lambda>0\right\}.
$$
The space $L^{\varphi}(\R^n)$ is a vector space and an important subspace is given by
$$
E^{\varphi}(\R^n)  := \left\{f \in M\left(\R^n\right):\ I^{\varphi} [\lambda f]<+\infty,\ \mbox{for\ every}\ \lambda>0\right\}.
$$
$E^{\varphi}(\R^n)$ is called the space of all finite elements of $L^{\varphi}(\R^n)$. It is easy to see that the following inclusions hold:
$
C_c(\R^n) \subset E^{\varphi}(\R^n) \subset L^{\varphi}(\R^n).
$
Clearly, functions belonging to $E^{\varphi}(\R^n)$ and $L^{\varphi}(\R^n)$ are not necessarily continuous. A norm on $L^{\varphi}(\R^n)$, called Luxemburg norm, can be defined by
$$
\left\|f\right\|_{\varphi}\ :=\ \inf \left\{\lambda>0:\ I^{\varphi}[f/\lambda] \miu \lambda\right\}, \hskip0.7cm (f \in L^{\varphi}(\R^n)).
$$

   We will say that a family of functions $(f_w)_{w>0} \subset L^{\varphi}(\R^n)$ is norm convergent to a function $f \in L^{\varphi}(\R^n)$, i.e., $\left\|f_w-f\right\|_{\varphi}\rightarrow 0$ for $w\to +\infty$, if and only if 
$\lim_{w \to +\infty} I^{\varphi}\left[\lambda(f_w-f)\right]\ =\ 0$, for every $\lambda>0$.
Moreover, an additional concept of convergence can be studied in Orlicz spaces: the "modular convergence". The latter induces a topology (modular topology) on the space $L^{\varphi}(\R^n)$ (\cite{MU1,BAMUVI}). 

  We will say that a family of functions $(f_w)_{w>0} \subset L^{\varphi}(\R^n)$ is modularly convergent to a function $f \in L^{\varphi}(\R^n)$ if
$
\lim_{w \to +\infty} I^{\varphi}\left[\lambda(f_w-f)\right]\ =\ 0$,
for some $\lambda>0$. Obviously, norm convergence implies modular convergence, while the converse implication does not hold in general. The modular and norm convergence are equivalent if and only if the $\varphi$-function $\varphi$ satisfies the $\Delta_2$-condition, see e.g., \cite{MU1,BAMUVI}. Finally, as last basic property of Orlicz spaces, we recall the following.
\begin{lemma}[\cite{BAMA}] \label{lemma1}
The space $C_c(\R^n)$ is dense in $L^{\varphi}(\R^n)$ with respect to the modular topology, i.e., for every $f \in L^{\varphi}(\R^n)$ and for every $\ep>0$ there exists a constant $\lambda>0$ and a function $g \in C_c(\R^n)$ such that $I^{\varphi}[\lambda(f-g)]<\ep$.
\end{lemma}


\section{The sampling Kantorovich operators}

In this section we recall the definition of the operators with which we will work. We will denote by 
$t_{\underline{k}}=\left(t_{k_1},...,t_{k_n}\right)$ a vector where each element $(t_{k_i})_{k_i \in \Z}$, $i=1,...,n$ is a sequence of real numbers with $-\infty < t_{k_i}<t_{k_{i+1}}<+\infty$, $\lim_{k_i \to \pm \infty}t_{k_i} = \pm 	\infty$, for every $i=1, ..., n$, and such that there exists $\Delta$, $\delta>0$ for which $\delta \leq \Delta_{k_i}:= t_{k_{i+1}}-t_{k_{i}} \leq  \Delta$, for every $i=1, ... , n$.
Note that, the elements of $(t_{k_i})_{k_i \in \Z}$ are not necessary equally spaced.  In what follows, we will identify with the symbol $\Pi^n$ the sequence $(t_{\underline{k}})_{\kk \in \Z^n}$.

  A function $\chi: \R^n \to \R$ will be called a kernel if it satisfies the following properties:
\begin{itemize}
	\item[$(\chi1)$] $\chi \in L^1(\R^n)$ and is bounded in a neighborhood of $\underline{0} \in \R^n$; 
	\item[$(\chi 2)$] for every $\underline{u} \in \R^n$, \hskip0.2cm $\displaystyle \sum_{\underline{k} \in \Z^n} \chi(\underline{u} - t_{\underline{k}})\ =\ 1$;
	\item[$(\chi 3)$] for some $\beta>0$,
$$
m_{\beta, \Pi^n }(\chi)\ =\ \sup_{\underline{u} \in \R^n}\sum_{\underline{k} \in \Z^n}\left|\chi(\underline{u}-t_{\underline{k}})\right|\cdot \left\|\underline{u}-t_{\underline{k}}\right\|^{\beta}_2\ <\ +\infty,
$$
where $\| \cdot \|_2$ denotes the usual Euclidean norm. 
\end{itemize}
We now recall the definition of the linear multivariate sampling Kantorovich operators introduced in \cite{COVI}. Define:
\begin{equation} \label{KANTO}
(S_w f)(\underline{x})\ :=\ \sum_{\underline{k} \in \Z^n} \chi(w\underline{x}-t_{\underline{k}})\left[\frac{w^n}{A_{\underline{k}}} \int_{R_{\underline{k}}^w}f(\underline{u})\ d\underline{u}\right], \hskip0.7cm (\underline{x} \in \R^n),
\end{equation}
where $f: \R^n \to \R$ is a locally integrable function such that the above series is convergent for every $\underline{x} \in \R^n$, where
\begin{displaymath}
R_{\underline{k}}^w\ :=\ \left[\frac{t_{k_1}}{w},\frac{t_{k_1+1}}{w}\right]\times\left[\frac{t_{k_2}}{w},\frac{t_{k_2+1}}{w}\right]\times...\times\left[\frac{t_{k_n}}{w},\frac{t_{k_n+1}}{w}\right],
\end{displaymath}
$w>0$ and $A_{\underline{k}} = \Delta_{k_1} \cdot \Delta_{k_2} \cdot...\cdot \Delta_{k_n}$, $\underline{k} \in \Z^n$. The operators in (\ref{KANTO}) have been introduced in \cite{BABUSTVI} in the univariate setting.
\begin{remark} \label{remark1} \rm
(a) Under conditions $(\chi 1)$ and $(\chi 3)$, the following properties for the kernel $\chi$ can be proved:
$$
m_{0, \Pi^n}(\chi)\ :=\ \sup_{\uu \in \R^n} \sum_{\kk \in \Z^n} \left|\chi(\uu-\tk)\right|\ <\ +\infty,
$$
and, for every $\gamma >0$
\begin{equation} \label{eee}
\lim_{w \to +\infty} \sum_{\left\|w \uu-\tk\right\|_2 >\gamma w}\left|\chi(w \uu - \tk)\right|\ =\ 0,
\end{equation}
uniformly with respect to $\uu \in \R^n$, see \cite{COVI}.
\vskip0.2cm
\noindent (b) By (a), we obtain that $S_w f$ with $f \in L^{\infty}(\R^n)$ are well-defined. Indeed,
\begin{displaymath}
\left|(S_w f)(\xx)\right|\ \miu\ m_{0, \Pi^n}(\chi)\left\|f\right\|_{\infty}\ < +\infty,
\end{displaymath}
for every $\xx \in \R^n$, i.e. $S_w : L^{\infty}(\R^n)\rightarrow L^{\infty}(\R^n)$.
\end{remark}
\begin{remark} \label{remark3} \rm
Note that, in the one-dimensional setting, choosing $t_{k}=k$, for every $k \in \Z$, condition $(\chi 2)$ is equivalent to 
\begin{displaymath}
\widehat{\chi}(k) :=\ \left\{
\begin{array}{l}
0, \hskip0.5cm k \in \Z\setminus \left\{0\right\}, \\
1, \hskip0.5cm k=0,
\end{array}
\right.
\end{displaymath}
where $\widehat{\chi}(v):=\int_{\R}\chi(u)e^{-ivu}\ du$, $v \in \R$, denotes the Fourier transform of $\chi$; see \cite{BUNE,BABUSTVI,COVI,COVI3}.
\end{remark}


\section{Convergence results}

In this section, we show the main approximation results for the multivariate sampling Kantorovich operators. In \cite{COVI}, the following approximation theorem for our operators has been proved. 
\begin{theorem} \label{th1}
Let $f \in C^0(\R^n)$. Then, for every $\underline{x} \in \R^n$,
\begin{displaymath}
\lim_{w \to +\infty} (S_w f)(\underline{x})\ =\ f(\underline{x}).
\end{displaymath}
In particular, if $f \in C(\R)$, then
\begin{displaymath}
\lim_{w \to +\infty} \left\|S_w f -f\right\|_{\infty}\ =\ 0.
\end{displaymath}
\end{theorem}
\begin{proof}
Here we highlight the main points of the proof.

\noindent Let $f \in C^0(\R^n)$ and $\xx \in \R^n$ be fixed. By the continuity of $f$ we have that for every fixed $\ep>0$ there exists $\gamma>0$ such that $|f(\xx)-f(\uu)|<\ep$ for every $\| \xx -\uu\|_2 \miu \gamma$, $\uu \in \R^n$. Then, by $(\chi 2)$ we obtain:
$$
\left|(S_w f)(\xx)-f(\xx)\right|\ \miu\ \sum_{\kk \in \Z^n} \left|\chi(w \xx - \tk)\right| \frac{w^n}{A_{\kk}}\int_{R^w_{\kk}}\left|f(\uu)-f(\xx)\right|\, d\uu\ 
$$
\vskip-0.5cm
\begin{eqnarray*}
&=& (\sum_{\left\|w \xx-\tk\right\|_2 \miu w \gamma/2}+\sum_{\left\|w \xx-\tk\right\|_2 > w \gamma/2})\left|\chi(w \xx - \tk)\right|\frac{w^n}{A_{\kk}}\int_{R^w_{\kk}}\left|f(\uu)-f(\xx)\right|\ d\uu \\
&=&\ I_1+I_2.
\end{eqnarray*}
For $\uu \! \in \! R^w_{\kk}$ and $\left\|w \xx-\tk\right\|_2\! \miu \! w \gamma/2$ we have $\|\uu-\xx\|_2\! \miu \! \gamma$ for  $w\! > \!0$ sufficiently large, then by the continuity of $f$ we obtain $I_1 \miu m_{0,\Pi^n}(\chi) \ep$ (see Remark \ref{remark1} (a)). Moreover, by the boundedness of $f$ and (\ref{eee}) we obtain $I_2 \miu 2\|f\|_{\infty} \ep$ for $w>0$ sufficiently large, then the first part of the theorem follows since $\ep>0$ is arbitrary. The second part of the theorem follows similarly replacing $\gamma>0$ with the parameter of the uniform continuity of $f$. 
\end{proof}

  In order to obtain a modular convergence result, the following norm-convergence theorem for the sampling Kantorovich operators (see \cite{COVI}) can be formulated.
\begin{theorem} \label{norm_conv} 
Let $\varphi$ be a convex $\varphi$-function. For every $f \in C_c(\R^n)$ we have
$$
\lim_{w \to +\infty} \left\|S_w f - f \right\|_{\varphi}\ =\ 0.
$$
\end{theorem}
Now, we recall the following modular continuity property for $S_w$, useful to prove the modular convergence for the above operators in Orlicz spaces.
\begin{theorem} \label{mod_cont} 
Let $\varphi$ be a convex $\varphi$-function. For every $f \in L^{\varphi}(\R^n)$ there holds
$$
I^{\varphi}[\lambda S_w f]\ \miu\ \frac{\left\|\chi\right\|_1}{\delta^n\cdot m_{0,\Pi^n}(\chi)}I^{\varphi}[\lambda m_{0,\Pi^n}(\chi)f],
$$
for some $\lambda>0$. In particular, $S_w$ maps $L^{\varphi}(\R^n)$ in $L^{\varphi}(\R^n)$.
\end{theorem}
\noindent Now, the main result of this section follows (see \cite{COVI}).
\begin{theorem} \label{th2}
Let $\varphi$ be a convex $\varphi$-function. For every $f \in L^{\varphi}(\R^n)$, there exists $\lambda>0$ such that
$$
\lim_{w \to +\infty} I^{\varphi}[\lambda (S_w f-f)]\ =\ 0.
$$
\end{theorem}
\begin{proof}
Let $f \in L^{\varphi}(\R^n)$ and $\ep>0$ be fixed. By Lemma \ref{lemma1}, there exists $\overline{\lambda}>0$ and $g \in C_c(\R^n)$ such that $I^{\varphi}[\overline{\lambda} (f-g)]\!<\! \ep$. Let now $\lambda>0$ such that $3\lambda(1 + m_{0,\Pi^n}(\chi)) \miu \overline{\lambda}$. By the properties of $\varphi$ and Theorem \ref{mod_cont}, we have
\begin{eqnarray*}
&& \hskip-0.7cm I^{\varphi}[\lambda(S_w f-f)]\ \miu\ I^{\varphi}[3\lambda(S_w f - S_w g)]+I^{\varphi}[3\lambda(S_w g - g)] + I^{\varphi}[3\lambda(f-g)]\\
&\miu&\ \frac{1}{m_{0,\Pi^n}(\chi)\cdot \delta^n}\left\|\chi \right\|_1 I^{\varphi}[\overline{\lambda} (f-g)] + I^{\varphi}[3\lambda(S_w g - g)] + I^{\varphi}[\overline{\lambda} (f-g)] \\
&<&\ \left(\frac{1}{m_{0,\Pi^n}(\chi) \cdot \delta^n}\left\|\chi \right\|_1+1\right) \ep + I^{\varphi}[3\lambda(S_w g - g)].
\end{eqnarray*}
The assertion follows from Theorem \ref{norm_conv}.
\end{proof} 
The setting of Orlicz spaces allows us to give a unitary approach for the reconstruction since we may obtain convergence results for particular cases of Orlicz spaces. For instance, choosing $\varphi(u)=u^p$, $1 \leq p < +\infty$, we have that $L^{\varphi}(\R^n) = L^p(\R^n)$ and $I^{\varphi}[f]=\|f\|^p_{p}$, where $\| \cdot \|_p$ is the usual $L^p$-norm. Then, from Theorem \ref{mod_cont} and Theorem \ref{th2} we obtain the following corollary.
\begin{corollary}
For every $f \in L^p(\R^n)$, $1 \leq p < +\infty$, 
the following inequality holds:
\begin{displaymath}
\left\|S_w f\right\|_p\ \miu\ \delta^{-n/p}\, [m_{0,\Pi^n}(\chi)]^{(p-1)/p}\, \left\|\chi\right\|^{1/p}_1\, \left\|f\right\|_p.
\end{displaymath}
Moreover, we have:
$$
\lim_{w \to +\infty} \|S_w f-f\|_p\ =\ 0.
$$
\end{corollary}
The corollary above, allows us to reconstruct $L^p$-signals (in $L^p$-sense), therefore signals/images not necessarily continuous. Other examples of Orlicz spaces for which the above theory can be applied can be found e.g., in \cite{MUORL,MU1,BAMUVI,BABUSTVI,COVI}. The theory of sampling Kantorovich operators in the general setting of Orlicz spaces allows us to obtain, by means of a unified treatment, several applications in many different contexts.


\section{Examples of special kernels} \label{sec5}

One important fact in our theory is the choice of the kernels, which influence the order of approximation that can be achieved by our operators (see e.g. \cite{COVI3} in one-dimensional setting).

  For instance, one  can take into consideration  {\em radial kernels}, i.e., functions for which the value depends on the Euclidean norm of the argument only. Example of such a kernel can be given, for example, by the Bochner-Riesz kernel, defined as follows $b^{\alpha}(\underline{x}):=2^{\alpha}\Gamma(\alpha+1)\|\underline{x}\|_2^{-(n/2)+\alpha}{\cal B}_{(n/2)+\alpha}(\|\underline{x}\|_2)$,
for $\underline{x} \in \R^n$, where $\alpha > (n-1)/2$, ${\cal B}_{\lambda}$ is the Bessel function of order $\lambda$ and $\Gamma$ is the Euler function. For more  details about this matter,  see e.g. \cite{BUFIST}. 

  To construct, in general, kernels satisfying all the assumptions $(\chi_i)$, $i=1,2,3$ is not very easy.

 For this reason, here we show a procedure useful to construct examples using product of univariate kernels, see e.g. \cite{BUFIST,COVI,COVI2}. In this case, we consider the case of uniform sampling scheme, i.e., $t_{\underline{k}}=\underline{k}$.

 Denote by $\chi_1, ..., \chi_n$, the univariate functions $\chi_i : \R \to \R$, $\chi_i \in L^1(\R)$, satisfying
the following assumptions:
\begin{equation} \label{uno-dim-}
m_{\beta,\Pi^1}(\chi_i)\ :=\ \sup_{u \in \R}\sum_{k \in \Z}\left|\chi_i(u-k)\right| |u-k|^{\beta}\ <\ +\infty,
\end{equation}
for some $\beta >0$, $\chi_i$ is bounded in a neighborhood of the origin and 
\begin{equation} \label{sing}
\sum_{k \in \Z}\chi_i(u-k) =1,
\end{equation}
for every $u \in \R$, for $i=1,...,n$. 
Now, setting $\chi(\underline{x}) := \prod_{i=1}^n\chi_i(x_i)$, $\underline{x}=(x_1,...,x_n) \in \R^n$,
it is easy to prove that $\chi$ is a multivariate kernel for the operators $S_w$ satisfying all the assumptions of our theory, see e.g., \cite{BUFIST,COVI}. 

  As a first example, consider the univariate Fej\'{e}r's kernel defined by $F(x)\ :=\ \frac{1}{2}\mbox{sinc}^2\left(\frac{x}{2}\right)$, $x \in \R$, where the sinc function is given by
\begin{displaymath}
\mbox{sinc}(x)\ :=\ \left\{
\begin{array}{l}
\disp \frac{\sin \pi x}{\pi x}, \hskip1cm x \in \R\setminus \left\{0\right\}, \\
\hskip0.5cm 1, \hskip1.5cm x=0.
\end{array}
\right.
\end{displaymath}
Clearly, $F$ is bounded, belongs to $L^1(\R)$ and satisfies the moment conditions $(\ref{uno-dim-})$ for $\beta = 1$, as shown in \cite{BUNE,BABUSTVI,COVI}. 
Furthermore, taking into account that the Fourier transform of $F$ is given by (see \cite{BUNE})
\begin{displaymath}
\widehat{F}(v) :=\ \left\{
\begin{array}{l}
1-|v/\pi|,\ \hskip0.5cm |v|\miu \pi, \\
0,\ \hskip1.85cm |v|>\pi,
\end{array}
\right.
\end{displaymath}
we obtain by Remark \ref{remark3} that condition $(\ref{sing})$ is fulfilled. Then, we can define $\disp \mathcal{F}_n(\xx)= \prod^n_{i=1}F(x_i)$, $\underline{x}=(x_1,...,x_n) \in \R^n,$ the multivariate Fej\'{e}r's kernel, satisfying the condition upon a multivariate kernel. The Fej\'{e}r's kernel $F$ and the bivariate Fej\'{e}r's kernel $F(x)\cdot F(y)$ are plotted in Figure \ref{fig1}.
\begin{figure}
\centering
\includegraphics[scale=0.22]{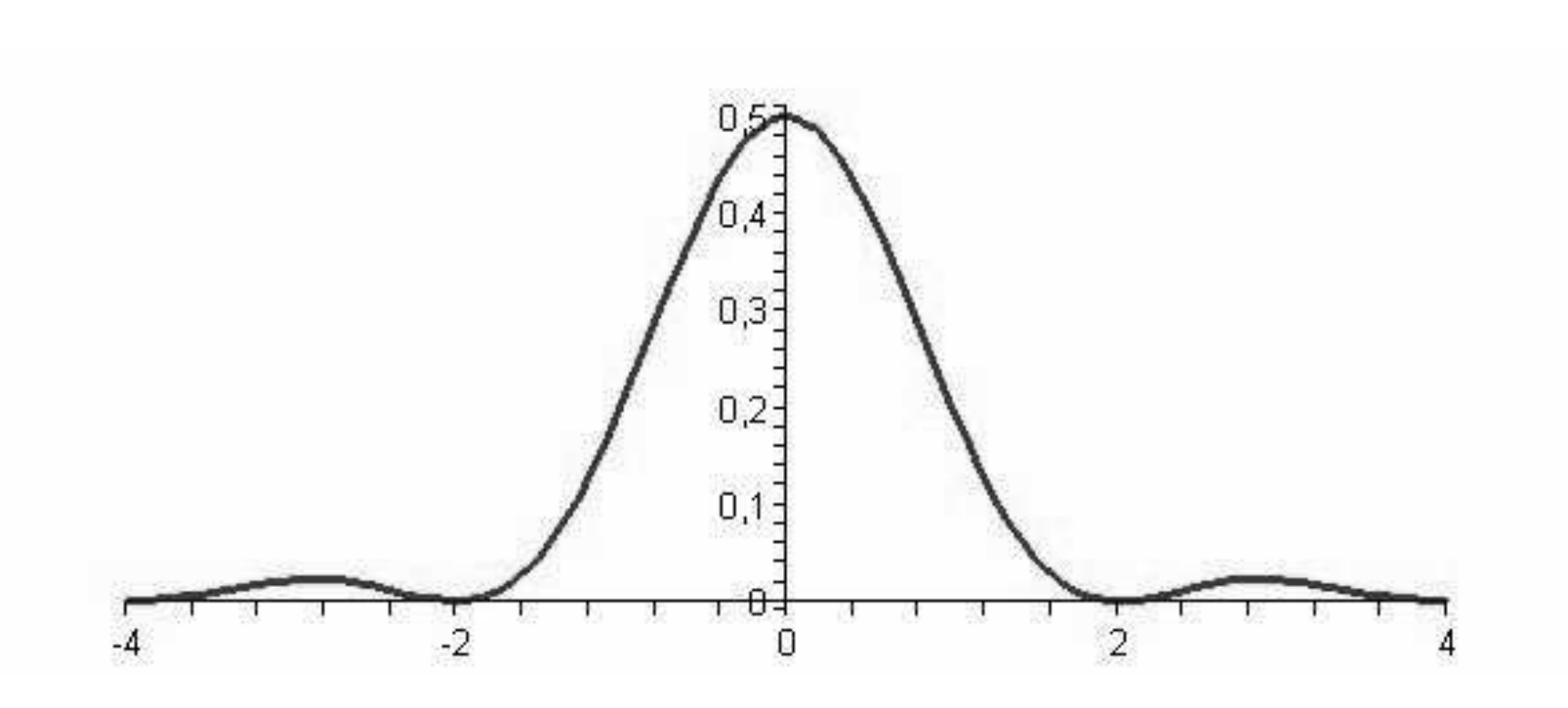}
\hskip0.8cm
\includegraphics[scale=0.22]{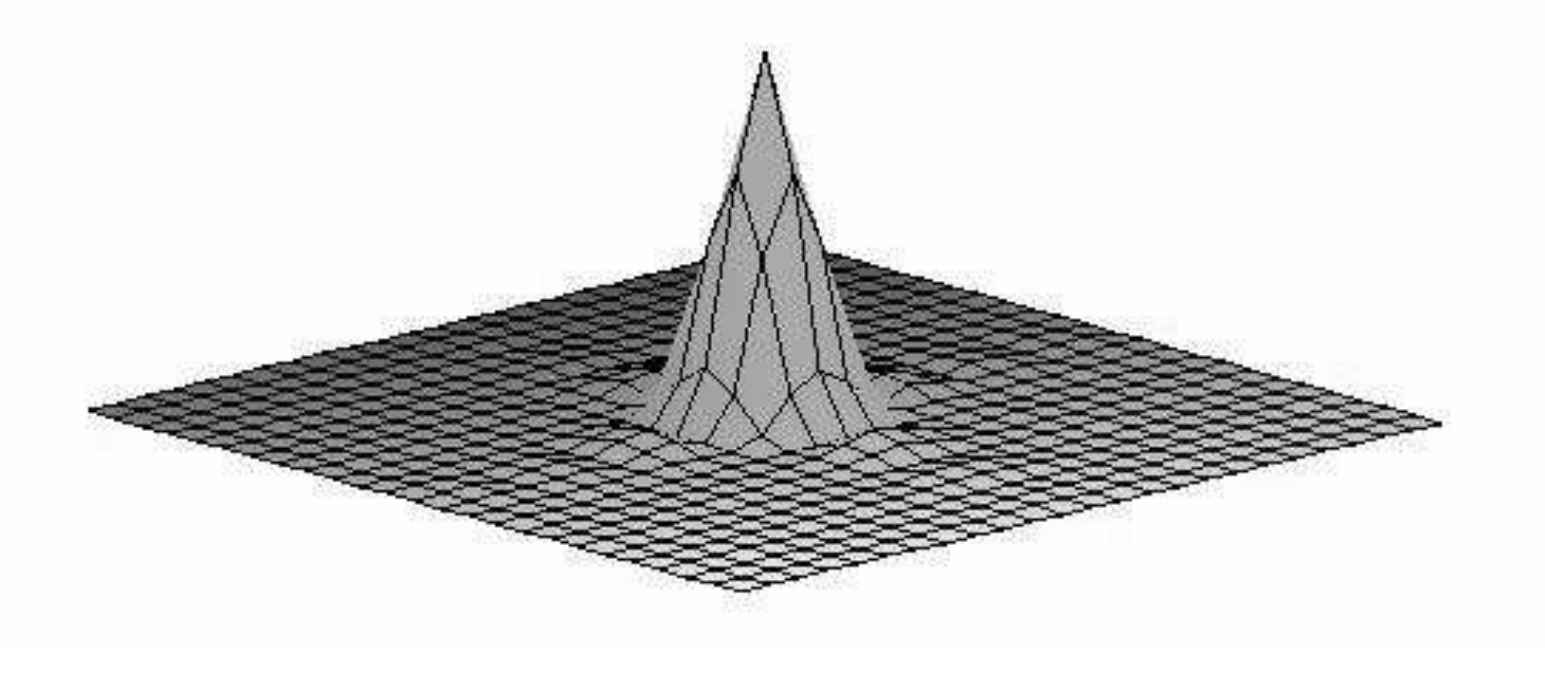}
\caption{\small{The Fej\'{e}r's kernel $F$ (left) and the bivariate Fej\'{e}r's kernel ${\cal F}(x,y)$ (right).}} \label{fig1}
\end{figure}

  The Fej\'{e}r's kernel $\mathcal{F}_n$ is an example of kernel with unbounded support, then to evaluate our sampling Kantorovich series at any given $\xx \in \R^n$, we need of an infinite number of mean values $w^n \int_{R^w_{\kk}}f(\uu)\ d\uu$. However, if the function $f$ has compact support, this problem does not arise. In case of function having unbounded support the infinite sampling series must be truncated to a finite one, which leads to the so-called truncation error. In order to avoid the truncation error, one can take kernels $\chi$ with bounded support. Remarkable examples of kernels with compact support, can be constructed using the well-known central B-spline of order $k \in \N$, defined by
$$
M_k(x) :=\ \frac{1}{(k-1)!}\sum^k_{i=0}(-1)^i \left(\begin{array}{l} \!\! 
k\\
\hskip-0.1cm i
\end{array} \!\! \right)
\left(\frac{k}{2}+x-i\right)^{k-1}_+,
$$
where the function $(x)_+ := \max\left\{x,0\right\}$ denotes the positive part of $x \in \R$ (see \cite{BABUSTVI,VIZA1,COVI}). The central B-spline $M_3$ and the bivariate B-spline kernel $M_3(x)\cdot M_3(y)$ are plotted in Figure \ref{fig2}.
\begin{figure}
\centering
\includegraphics[scale=0.22]{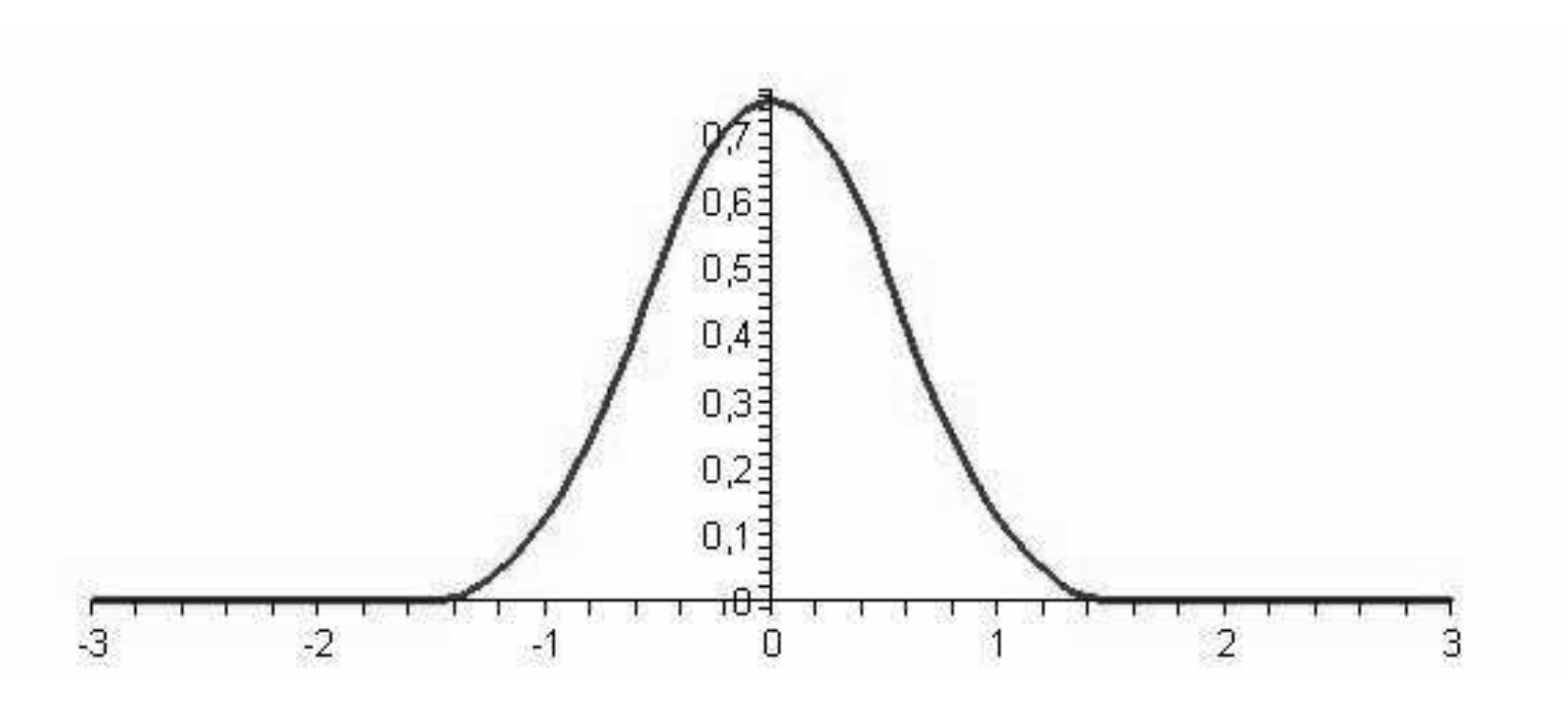}
\hskip0.7cm
\includegraphics[scale=0.22]{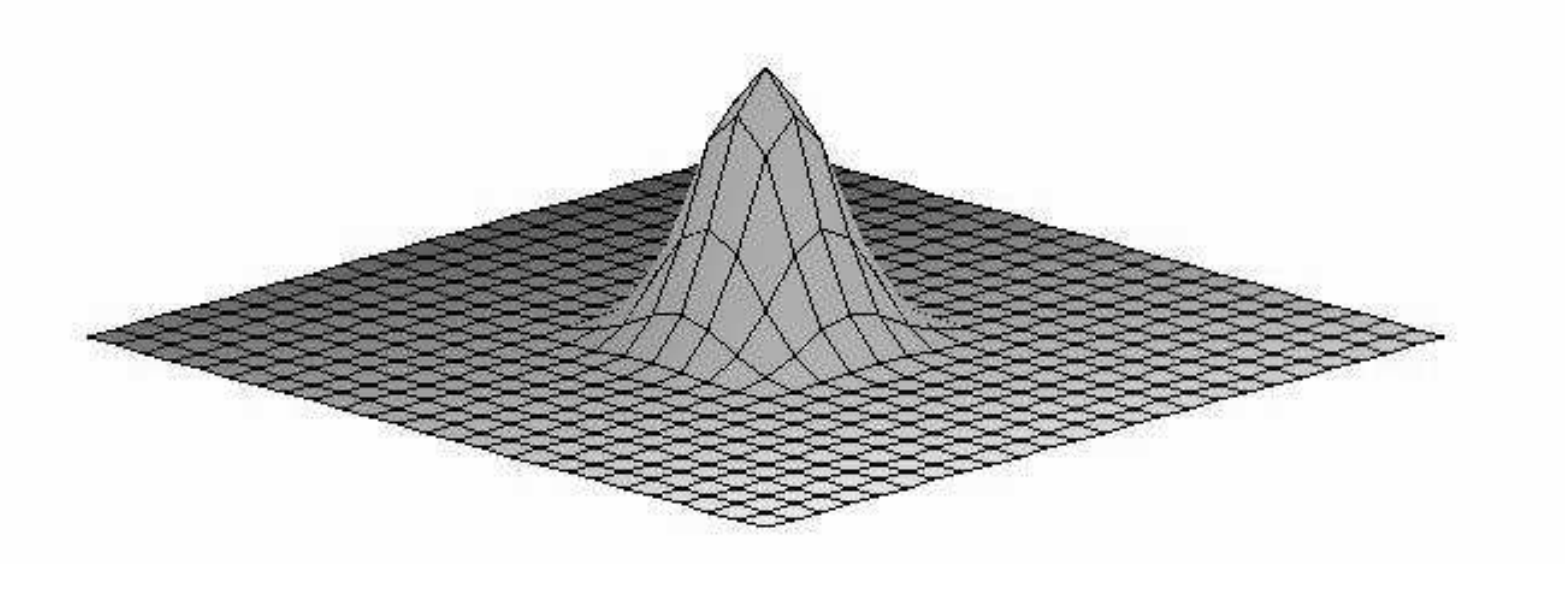}
\caption{\small{The B-spline $M_3$ (left) and the bivariate B-spline ${\cal M}^2_3(x,y)$ (right).}} \label{fig2}
\end{figure}
We have that, the Fourier transform of $M_k$ is given by
$\widehat{M_k}(v)\ :=\ \mbox{sinc}^n\left( \frac{v}{2 \pi} \right)$, $v \in \R$, and then, if we consider the case of the uniform spaced sampling scheme, condition (\ref{sing}) is satisfied by Remark \ref{remark3}. Clearly, $M_k$ are bounded on $\R$, with compact support $[-n/2,n/2]$, and hence $M_k \in L^1(\R)$, for all $k \in \N$. Moreover, it easy to deduce that condition $(\ref{uno-dim-})$ is fulfilled for every $\beta>0$, see \cite{BABUSTVI}.
Hence $\mathcal{M}^n_k(\xx)\ :=\ \prod^n_{i=1}M_k(x_i)$, $\underline{x}=(x_1,...,x_n) \in \R^n$,
is the multivariate B-spline kernel of order $k \in \N$.

  Finally, other important examples of univariate kernels are given by the Jackson-type kernels, defined by
$J_k(x)=c_k\, \mbox{sinc}^{2k}\left(\frac{x}{2k\pi\alpha}\right)$, $x \in  \R$,
with $k \in \N$, $\alpha \geq 1$, where the normalization coefficients $c_k$ are given by
$c_k\ :=\ \left[ \int_{\R} \mbox{sinc}^{2k}\left(\frac{u}{2 k \pi \alpha} \right) \, du \right]^{-1}$.
Since $J_k$ are band-limited to $[-1/\alpha,\, 1/\alpha]$, i.e., their Fourier transform vanishes outside this interval, condition $(\chi 2)$ is satisfied by Remark \ref{remark3} and $(\ref{uno-dim-})$ is satisfied since $J_k(x)=\mathcal{O}(|x|^{-2k})$, as $x \to \pm \infty$, $k \in \N$. In similar manner, by the previous procedure we can construct a multivariate version of Jackson kernels. For more details about Jackson-type kernels, and for others useful examples of kernels see e.g. \cite{BUNE,BAMUVI,BABUSTVI0,BABUSTVI}.
%


\section{The sampling Kantorovich algorithm for image reconstruction}

In this section, we show applications of the multivariate sampling Kantorovich operators to image reconstruction. First of all, we recall that every bi-dimensional gray scale image $A$ is represented by a suitable matrix and can be modeled as a step function $I$, with compact support, belonging to $L^p(\R^2)$, $1 \leq p <+\infty$. The definition of $I$ arise naturally as follows:
$$
I(x,y)\ :=\ \sum^{m}_{i=1}\sum^{m}_{j=1}a_{ij} \cdot \textbf{1}_{ij}(x,y), \hskip0.3cm (x,y) \in \R^2,
$$
where ${\bf 1}_{ij}(x,y)$, $i,j =1,2,...,m$, are the characteristics functions of the sets $(i-1,\ i]\times(j-1,\ j]$ (i.e. ${\bf 1}_{ij}(x,y)=1$, for $(x,y)\ \in\ (i-1,\ i]\times(j-1,\ j]$ and ${\bf 1}_{ij}(x,y)=0$ otherwise). 

  The above function $I(x,y)$ is defined in such a way that, to every pixel $(i,j)$ it is associated the corresponding gray level $a_{ij}$. 

  We can now consider approximation of the original image by the bivariate sampling Kantorovich operators $(S_w I)_{w>0}$ based upon some kernel $\chi$. 

  Then, in order to obtain a new image (matrix) that approximates in $L^p$-sense the original one, it is sufficient to sample $S_w I$ (for some $w>0$) with a fixed sampling rate. In particular, we can reconstruct the approximating images (matrices) taking into consideration different sampling rates and this is possible since we know the analytic expression of $S_w I$.

 If the sampling rate is chosen higher than the original sampling rate, one can get a new image that has a better resolution than the original one's. The above procedure has been implemented by using MATLAB and tools of matrix computation in order to obtain an optimized algorithm based on the multivariate sampling Kantorovich theory.

  In the next sections, examples of thermographic images reconstructed by the sampling Kantorovich operators will be given to show the main applications of the theory to civil engineering. In particular, a real world case-study is analyzed in term of modal analysis to obtain a model useful to study the response of the building under seismic action.
%

%
%
%
%
%
%
%
%
%
%
%

\section{An application of thermography to civil engineering}

In the present application, thermographic images will be used to locate the resisting elements and to define their dimensions, and moreover to investigate the actual texture of the masonry wall, i.e., the arrangement of blocks (bricks and/or stones) and mortar joints. 
In general, thermographic images have a resolution too low to accurately analyze  the texture of the masonries,  therefore it has to be increased by means of suitable tools.

In order to obtain a consistent separation of the phases, that is a correct identification of the pixel which belong to the blocks and those who belong to mortar joints, the image is converted from gray-scale representation to black-and-white (binary) representation by means of an image texture algorithm, which employs techniques belonging to the field of digital image processing. 
The image texture algorithm, described in details in \cite{CACLGU}, leads to areas of white pixels identified as blocks and areas of black pixels identified as mortar joints.
However, the direct application of the image texture algorithm to the thermographic images  (see Figure~\ref{fig:original_ric}), can produce errors, as an incorrect separation between the bricks and the mortar (see Figure~\ref{fig:original_BW}). 
Therefore, we can use the sampling Kantorovich operators to process the thermographic images. 
In particular, here we used the operators $S_w$ based upon the bivariate Jackson-type kernel with $k=12$ (see Section \ref{sec5}) and the parameter $w=40$ (see Figure~\ref{fig:reconstructed_ric}). 
The application of the image texture algorithm produces a consistent separation of the phases (see Figure~\ref{fig:reconstructed_BW}).

\begin{figure}[!ht]
\centering
\subfigure[]{ \includegraphics[width = 0.2\textwidth]{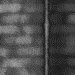} \label{fig:original_ric}}
\hskip0.15cm
\subfigure[]{ \includegraphics[width = 0.2\textwidth]{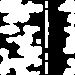} \label{fig:original_BW}}
\hskip0.15cm 
\subfigure[]{ \includegraphics[width = 0.2\textwidth]{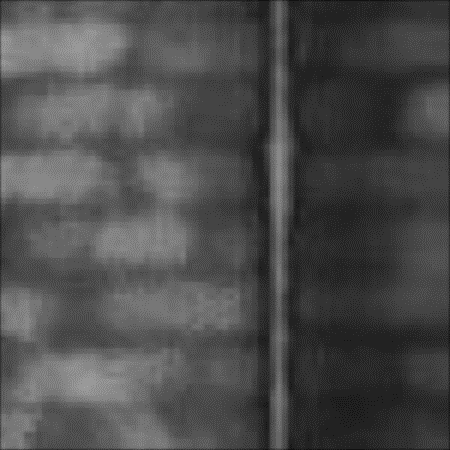} \label{fig:reconstructed_ric}}
\hskip0.15cm
\subfigure[]{ \includegraphics[width = 0.2\textwidth]{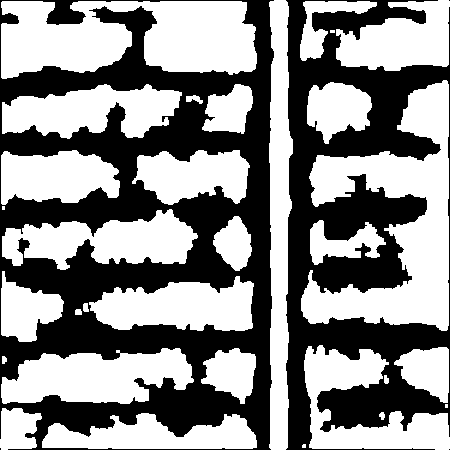} \label{fig:reconstructed_BW}} 
\caption{{\small (a) Original thermographic image ($75\times 75$ pixel resolution) and (b) its texture; (c) Reconstructed thermographic image ($450\times 450$ pixel resolution) and (d) its texture.}}
\label{fig:original}
\end{figure}

In order to perform structural analysis, the mechanical characteristics of an homogeneous material equivalent to the original heterogeneous material are sought. 
The equivalence is in the sense that, when subjected to the same boundary conditions (b.c.), the overall responses in terms of mean values of stresses and deformations are the same.
The equivalent elastic properties are estimated by means of the ``test-window'' method \cite{CLGU}.

\section{Case-study}

The image texture algorithm described previously  has been used to analyze a real-world case-study: a building  consisting of two levels and an attic, with a very simple architectural plan, a rectangle with sides 11 m and 11.4 m (see Figure~\ref{fig:pianta_prospetto}). 
The vertical structural elements consist of masonry walls. 
At first level the masonry walls have thickness of 40 cm with blocks made of  stones, while at second level the masonry is made of squared tuff block and it has thickness of about 35 cm. 
Both surfaces of the walls are plastered.
The slab can be assumed to be rigid in the horizontal plane and 
the building is placed in a medium risk seismic area. 
According to a preliminary visual survey, the structure does not have strong asymmetries between principal directions and the distribution of the masses are quite uniform.

\begin{figure}[!ht]
\centering
\includegraphics[width = 0.75\textwidth]{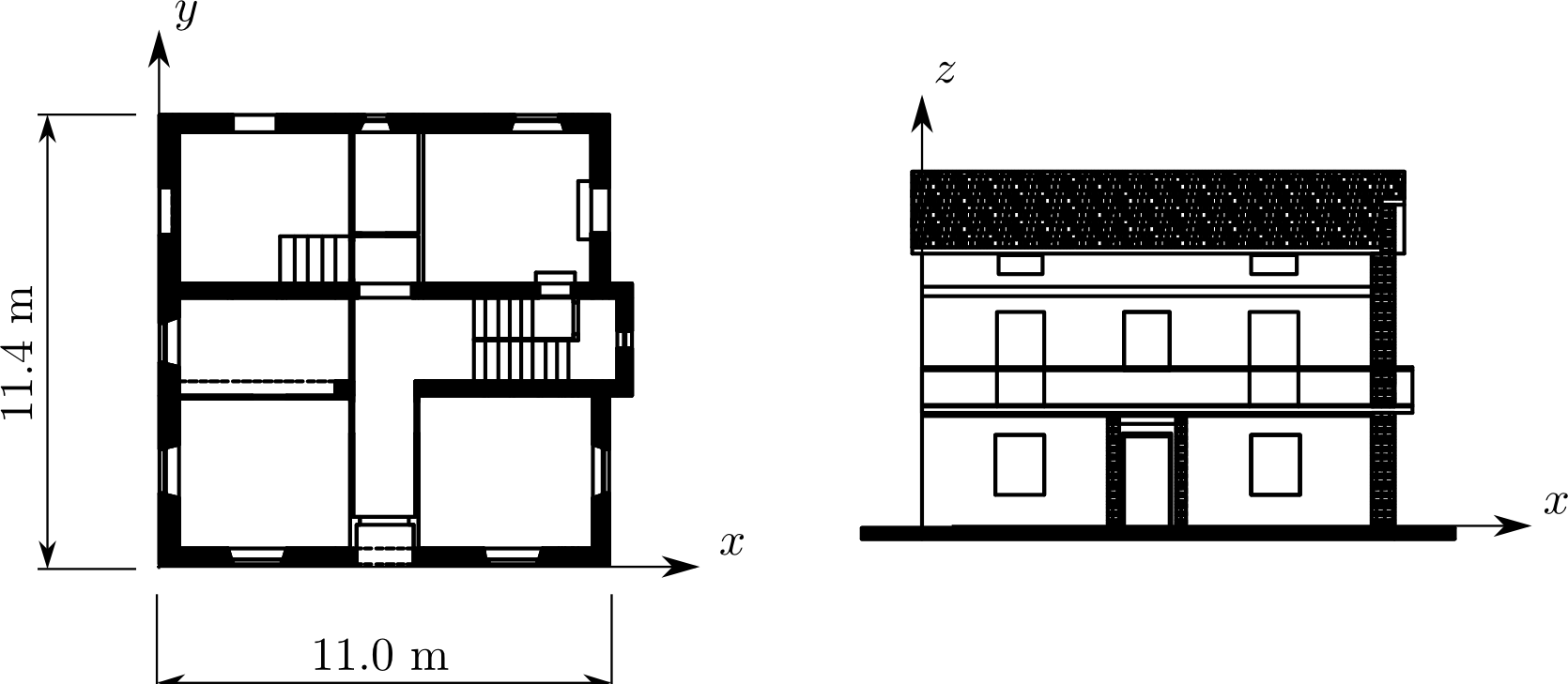} 
\caption{{\small Case-study building: plant (left) and main facade (right).}}
\label{fig:pianta_prospetto}
\end{figure}

The response of the building is evaluated by means of three different models.
In the first model, only information that can be gathered by the visual survey have been used, and the characteristics of material has been assume according to Italian Building Code. 
In the second model information concerning the actual geometry of vertical structural elements acquired by means of thermographic survey have been used; 
for example, the survey showed that the openings at ground level in the main facade were once larger and have been subsequently partially filled with non-structural masonry, and thus the dimensions of all the masonry walls in the facade have to be reduced.
Eventually, in the third model information about the actual texture of masonries of ground level (chaotic masonry, Fig.~\ref{fig:chaotic}) and of second level (periodic masonry, Fig.~\ref{fig:periodic}) have been used. 
In particular, the actual textures where established using the reconstructed thermographic images. 

\begin{figure}[!ht]
\centering
\subfigure[]{ \includegraphics[width = 0.2\textwidth]{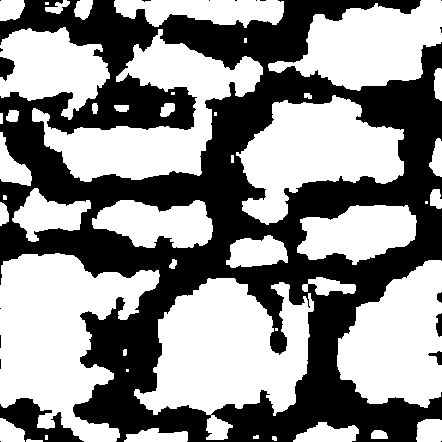} \label{fig:chaotic} }
\hskip0.8cm
\subfigure[]{ \includegraphics[width = 0.2\textwidth]{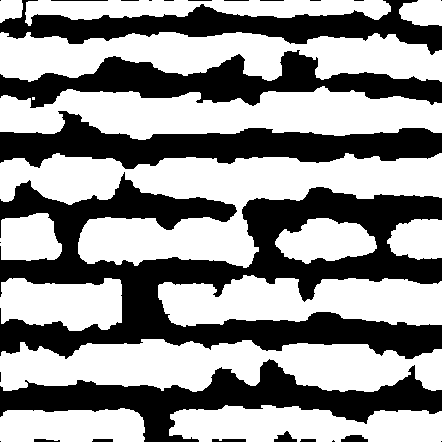} \label{fig:periodic}} 
\caption{{\small Texture of (a) ground level and (b) second level masonries.}}
\label{fig:masonries}
\end{figure}

The main characteristics of the models are reported in Table~\ref{tab:char_models}.
For the mechanical characteristic the Young's modulus $E$ and the shear modulus $G$ are shown.
In the third model, the following mechanical characteristics (Young's modulus $E$ and Poisson's ratio $\nu$) of the constituent phases have been used: for the stones of the ground level, $E = 25000\,\mathrm{N\,mm^{-2}}, \nu = 0.2$, for the bricks of the second level $E = 1700\,\mathrm{N\,mm^{-2}}, \nu = 0.2$, for the mortar of both levels  $E = 2500\,\mathrm{N\,mm^{-2}}, \nu = 0.2$.

\begin{table}[!ht]
\center
{\small
\begin{tabular}{c|c|cc|cc}
\hline
 & & \multicolumn{2}{|c}{ground level} & \multicolumn{2}{|c}{second level}\\
 & geometry & $E$ & $G$ & $E$ & $G$ \\
 & & $\mathrm{N\,mm^{-2}}$ & $\mathrm{N\,mm^{-2}}$ & $\mathrm{N\,mm^{-2}}$ & $\mathrm{N\,mm^{-2}}$ \\
\hline
 Model \#1 & visual survey & 3346 & 1115 & 1620 & 540 \\
 Model \#2 & thermogr. survey & 3346 & 1115 & 1620 & 540 \\
 Model \#3 & thermogr. survey & 7050 & 2957 & 1996 & 833 \\
 \hline\end{tabular}
}
\caption{{\small Geometry and masonries' mechanical characteristics in the models.}}
\label{tab:char_models}
\end{table}

The behavior under seismic actions is estimated by means of modal analysis, using a commercial code based on the Finite Element Method.
The periods of the first three modes and the corresponding mass participating ratio  for the two principal direction of seismic action are reported in Tab.~\ref{tab:periodi_masse}.
For all the models, the first mode is along the $y$ axis, the second is along the $x$ axis, and the third is mainly torsional. 
Anyway, the second mode is not a pure translational one since the asymmetry in the walls distribution produces also a torsional component. 

\begin{table}[h!]
\center
{\small
\begin{tabular}{c|p{0.7cm}p{0.7cm}p{0.7cm}|p{0.7cm}p{0.7cm}p{0.7cm}|p{0.7cm}p{0.7cm}p{0.7cm}}
\hline
& \multicolumn{3}{|c|}{Periods}
 & \multicolumn{6}{|c}{Mass partecipating ratio} \\
 & \multicolumn{3}{|c|}{ }
  & \multicolumn{3}{|c|}{ $x$ direction} &    \multicolumn{3}{|c}{$y$ direction} \\
 & 1st mode & 2nd mode & 3rd mode  & 1st mode & 2nd mode & 3rd mode  & 1st mode & 2nd mode & 3rd mode\\
\hline
 Model \#1 & 0.22 & 0.15 & 0.13 & 0.00 & 0.64 & 0.13 & 0.73 & 0.00 & 0.00 \\
 Model \#2 & 0.22 & 0.14 & 0.13 & 0.00 & 0.51 & 0.27 & 0.73 & 0.01 & 0.00 \\
 Model \#3 & 0.18 & 0.13 & 0.11 & 0.00 & 0.46 & 0.26 & 0.68 & 0.01 & 0.00 \\
 \hline\end{tabular}
}
\caption{{\small Periods and mass participating ratio for two directions of seismic action of the first three modes .}}
\label{tab:periodi_masse}
\end{table}

As can be noted, the first two models have the same period for the fundamental mode in $y$ direction; nevertheless, in $x$ direction there is a slight difference due to the reduced dimensions of masonry walls width discovered by means of thermographic survey. In fact, the reduction of dimensions leads to a decrease of global stiffness while the total mass is almost the same.  
The periods of third model show a reduction of about 20\% due to  the greater value of equivalent elastic moduli estimated by means of homogenized texture.
As already noted, for seismic action in $x$ direction the structure response is dominated by second mode, with also a significant contribution from the third mode (which is torsional); for seismic action in $y$ the response is dominated by first mode.
It is also worth noting that Model \#3 shows a reduced mass participating ratio for the fundamental mode in each direction, and therefore a greater number of modes should be considered in the evaluation of seismic response in order to achieve a suitable accuracy.


\section{Concluding remarks}

The sampling Kantorovich operators and the corresponding MATLAB algorithm for image reconstruction are very useful to enhance the quality of thermographic images of portions of masonry walls. In particular, after the processing by sampling Kantorovich algorithm, the thermographic image has higher definition with respect to the original one and therefore, it was possible to estimate the mechanical characteristics of homogeneous materials equivalent to actual masonries, taking into account the texture (i.e., the arrangement of blocks and mortar joints). These materials were used to model the behavior of a case study under seismic action. This model has been compared with others constructed by well-know methods, using ``naked eyes'' survey and the mechanical parameters for materials taken from the Italian Building Code. Our method based on the processing of thermographic images by sampling Kantorovich operators enhances the quality of the model with respect to that based on visual survey only.

\noindent In particular the proposed approach allows to overcome some difficulties that arise when dealing with the vulnerability analysis of existing structures, which are: i) the knowledge of the actual geometry of the walls (in particular the identification of hidden doors and windows); ii) the identification of the actual texture of the masonry and the distribution of inclusions and mortar joints, and from this iii) the estimation of the elastic characteristics of the masonry.
It is noteworthy that, for item i) the engineer has usually limited knowledge, due to the lack of documentation, while for items ii) and iii) he usually use tables proposed in technical manuals and standards which however give large bounds in order to encompass the generality of the real masonries. Instead, the use of reconstruction techniques on thermographic images coupled with homogenization permits to reduce these latter uncertainty on the estimation of the mechanical characteristics of the masonry.

%
%
%
%
%
%
%

\vskip0.1cm
{\footnotesize
{\bf Authors affiliation:}
\vskip0.1cm
\begin{enumerate}
\item {\em Federico Cluni},  Department of Civil and Environmental Engineering, University of Perugia, Via G. Duranti, 93, 06125, Perugia, Italy.\\
Email: {\tt federico.cluni@unipg.it}.

\item {\em Danilo Costarelli}, Department of Mathematics and Physics, Section of Mathematics, University of Roma Tre, Largo San Leonardo Murialdo 1, 06146, Rome, Italy. Email: {\tt costarel@mat.uniroma3.it}.

\item {\em Anna Maria Minotti} and {\em Gianluca Vinti}, Department of Mathematics and Computer Sciences, University of Perugia, Via Vanvitelli 1, 06123, Perugia, Italy. Emails: {\tt annamaria.minotti@dmi.unipg.it} and {\tt gianluca.vinti@unipg.it}.
\end{enumerate}
}

\end{document}